\documentclass[twoside,11pt]{article}

\usepackage{amsmath, amssymb}
\usepackage{subcaption}
\usepackage{xfrac}
\usepackage[margin=3.5cm]{geometry}
\usepackage{amsthm}
\usepackage{url}
\theoremstyle{plain}
\newtheorem{theorem}{Theorem}[section]
\newtheorem{proposition}[theorem]{Proposition}
\newtheorem{lemma}[theorem]{Lemma}
\theoremstyle{definition}
\newtheorem{definition}[theorem]{Definition}
\theoremstyle{definition}
\newtheorem{example}[theorem]{Example}
\theoremstyle{remark}
\newtheorem{remark}[theorem]{Remark}
\author{David Eklund\thanks{DTU Compute, Richard Petersens Plads,
Building 321, DK-2800 Kgs. Lyngby, Denmark.} \\ \url{daek@kth.se}
\and S{\o}ren Hauberg\footnotemark[1] \\ \url{sohau@dtu.dk}}

\newcommand{\propref}[1]{Proposition~\ref{#1}}

\newcommand{\lemmaref}[1]{Lemma~\ref{#1}}

\newcommand{\exref}[1]{Ex.~\ref{#1}}
\newcommand{\secref}[1]{Sec.~\ref{#1}}

\newcommand{\remref}[1]{Remark~\ref{#1}}

\newcommand{\figref}[1]{Fig.~\ref{#1}}

\newcommand{\RR}{\mathbb{R}}
\newcommand{\EE}{\mathbb{E}}

\newcommand{\E}{\EE}
\newcommand{\id}[1]{\operatorname{Id}_{#1}}
\newcommand{\var}[1]{\operatorname{var}(#1)}
\newcommand{\covar}[2]{\operatorname{covar}(#1,#2)}
\newcommand{\SO}[2]{\operatorname{SO}(#1,#2)}

\renewcommand{\vec}[1]{#1}
\newcommand{\mat}[1]{#1}
\newcommand{\x}[0]{\vec{x}}

\newcommand{\z}[0]{\vec{z}}
\newcommand{\J}[0]{\mat{J}}

\newcommand{\dif}[1]{\mathrm{d}#1}

\newcommand{\inner}[2]{\langle#1,#2\rangle}

\newcommand{\T}[0]{^{\top}}
\newcommand{\xtx}[1]{#1\T \mkern-1.5mu\relax #1}
\def\N{\mathcal{N}}
\def\M{\mathcal{M}}

\def\Z{Z}

\def\X{X}

\begin{document}

\title{Expected path length on random manifolds}

\maketitle

\begin{abstract}
  Manifold learning seeks a low dimensional representation that
  faithfully captures the essence of data. Current methods can
  successfully learn such representations, but do not provide a
  meaningful set of operations that are associated with the
  representation. Working towards \emph{operational representation
    learning}, we endow the latent space of a large class of
  generative models with a random Riemannian metric, which provides us
  with elementary operators. As computational tools are unavailable
  for random Riemannian manifolds, we study deterministic
  approximations and derive tight error bounds on expected distances.
\end{abstract}

\section{Introduction} \label{sec:intro}
  \emph{Manifold learning} is one of the cornerstones of unsupervised
  learning.  Classical methods such as \emph{Isomap} \cite{isomap},
  \emph{Locally linear embeddings} \cite{lle}, \emph{Laplacian
    eigenmaps} \cite{Belkin:2003:LED} and more
  \cite{Scholkopf99:kernelprincipal, Donoho5591} all seek a
  low dimensional embedding of high dimensional data that preserves
  prespecified aspects of data.  Probabilistic methods often view the
  data manifold as governed by a latent variable along with a
  generative model that describes how the latent manifold is to be
  embedded in the data space. The common theme is the quest for a
  low dimensional representation that faithfully captures the data.
  
  Ideally, we want an \emph{operational representation}, that is we
  want to be able to make mathematically meaningful calculations with
  respect to the learned representation. It has been argued
  \cite{H18} that a good representation should at least support the
  following:
  \begin{itemize}
    \item \textbf{Interpolation:} given two points, a natural unique
      interpolating curve that follows the manifold should exist.
    \item \textbf{Distances:} the distance between two points should be well defined
      and informally
      reflect the amount of energy required to transform one point to another.
    \item \textbf{Measure:} the representation should be equipped with a measure
      under which integration is well defined for all points on the manifold.
  \end{itemize}
  These are elementary requirements of a representation, but most nonlinear manifold
  learning schemes do not imply or provide such operations.

In the sequel we will use the following notation. We denote by $\Z$
the $d$-dimensional \emph{representation} or \emph{latent space},
which is learned from data in the \emph{observation space}
$\X$. Latent points are denoted $\z_i \in \Z$, while corresponding
observations are $\x_i \in \X$.
  
  \begin{figure}
    \begin{center}
    \includegraphics[width=0.3\columnwidth]{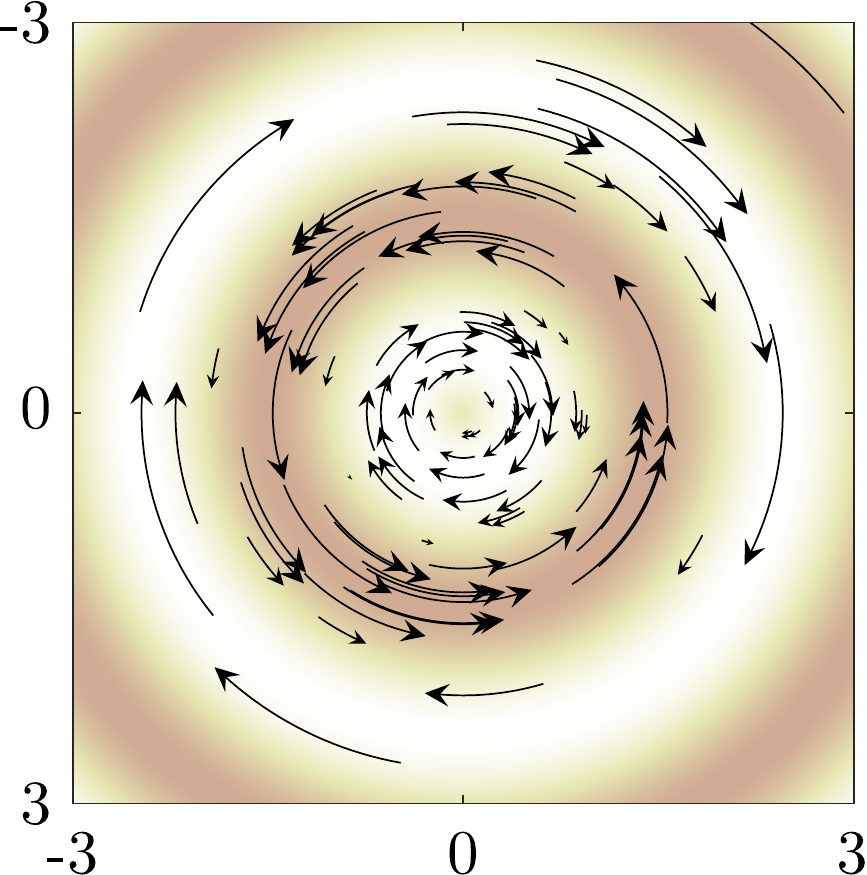}
    \hspace{0.07\columnwidth}
    \includegraphics[width=0.3\columnwidth]{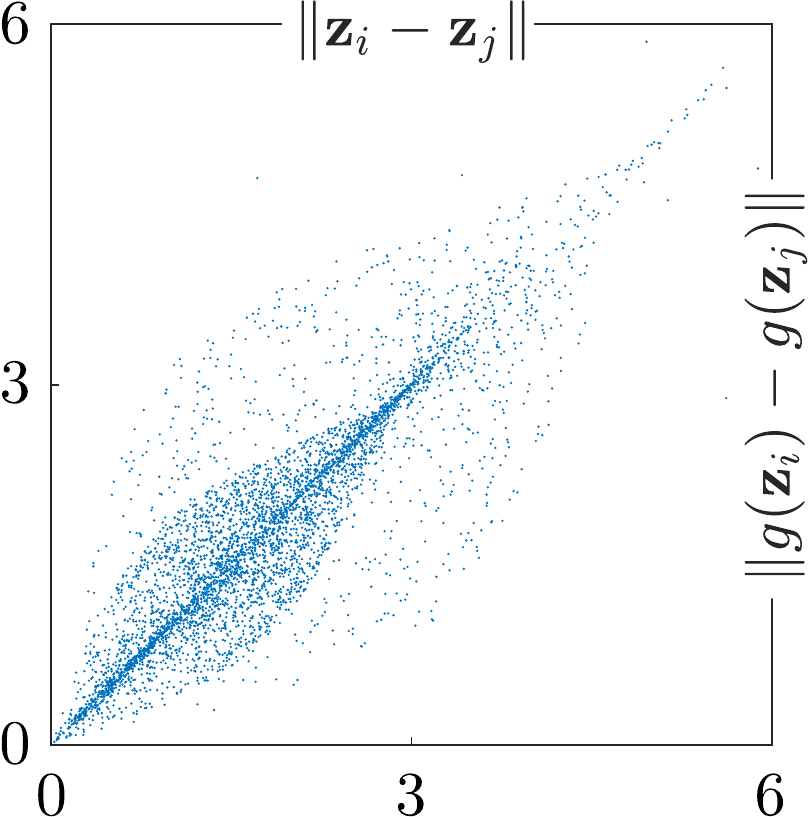}
    \end{center}
    \caption{Reparametrizations illustrated.  \emph{Left:} A
      ``swirling'' transformation of the latent space
      $\Z$. \emph{Right:} Pair wise distances between points before
      and after reparametrization; clearly the Euclidean distances
      change with reparametrizations.}
    \label{fig:swirl}    
  \end{figure}

  \textbf{Embedding methods} seek a low dimensional embedding
  $\z_{1:N} = \{\z_1, \ldots, \z_N\} \subset Z$ of the data $\x_{1:N}
  \subset \X$. These methods fundamentally only describe the data
  manifold at the points where data is observed and nowhere else.  As
  such, the low dimensional embedding space is only well defined at
  $\z_{1:N}$.  It is common to treat the low dimensional embedding
  space as equipped with the Euclidean metric, but this is generally a
  \emph{post hoc} assumption with limited grounding in the embedding
  method. Fundamentally, the learned representation space is a
  discrete space that does not lend itself to continuous
  interpolations.  Likewise, the most natural measure will only assign
  mass to the points $\z_{1:N}$, and any associated distribution will
  be discrete. It is not clear how this can naturally lead to an
  operational representation.
  
  \textbf{Generative models} estimate a set of low dimensional latent
  variables $\z_{1:N}$ along with a suitable mapping $f: \Z
  \rightarrow \X$ such that $f(\z) \approx \x$. It is, again, common
  to treat the latent space $\Z$ together with the Euclidean
  metric. However, this assumption is unwarranted. As an example,
  consider the \emph{variational autoencoder (VAE)}
  \cite{kingma:iclr:2014, rezende2014stochastic}, which seeks a
  representation in which $\z_{1:N}$ follow a unit Gaussian
  distribution.  Now consider a 2-dimensional latent space and the
  transformation $g(\z) = \mat{R}_{\theta} \z$, where
  $\mat{R}_{\theta}$ is a linear transformation that rotates points by
  $\theta(\z) = \sin(\pi \|\z\|)$. This is a smooth invertible
  transformation with the property that $\z \sim \N(\vec{0}, \mat{I})
  \Rightarrow g(\z) \sim \N(\vec{0}, \mat{I})$; see
  Fig.~\ref{fig:swirl}. If the latent variables $\z_{1:N}$ and the
  mapping $f$ is an optimal VAE, then $g(\z_{1:N})$ and $f \circ
  g^{-1}$ is equally optimal. Yet, the latent spaces $\Z$ and $g(\Z)$
  are quite different; Fig.~\ref{fig:swirl} shows the Euclidean
  distances between pairs of points of the latent space before and
  after applying $g$, for samples drawn from a unit Gaussian. Clearly,
  the transformed latent space is significantly different from the
  original space. As the VAE provides no guarantees as to which latent
  space is recovered, we must be careful when relying on the Euclidean
  latent space: distances between points are effectively arbitrary, as
  are straight line interpolations. Ideally, we want a representation
  that is invariant under such transformations, but current models do
  not have such properties.

  \textbf{In this paper}, we consider probabilistic latent variable
  models on the form $\x = f(\z)$ where $f$ is a smooth stochastic
  process. The latent space can then be endowed with a random
  Riemannian metric to ensure that the learned latent representation
  is operational as defined above. We consider a deterministic
  approximation to the random Riemannian metric, and provide tight
  approximation bounds for expected distances
  (\propref{prop:lengths}). The approximation is good when the data is
  high dimensional, which is often the case in machine learning
  applications. The analysis justifies the use of deterministic
  approximations, which in turn lead to computationally tractable
  algorithms.
  
  The paper is structured to first provide a short primer on (deterministic)
  Riemannian geometry (Sec.~\ref{sec:primer}). We then extend this class of geometries
  to the stochastic setting (Sec.~\ref{sec:stoch_riem_geom}), and provide our
  main theoretical contributions (Sec.~\ref{sec:rates}) that analyze to which extend
  stochastic manifolds are well approximated by deterministic ones. Our analysis
  holds for any smooth stochastic generative process, which we exemplify (Sec.~\ref{sec:empirical})
  with Gaussian process latent variable models \cite{L05}.
  
  \section{Riemannian manifolds}\label{sec:primer}
  A $d$-dimensional manifold $\M$ embedded in $\RR^n$ with $d < n$ is
  a topological space in which each point $x \in \M$ has a
  neighborhood that is homeomorphic to $\RR^d$
  \cite{gallot1990riemannian}. We may think of $\M$ as a smooth
  (nonlinear) surface in space that does not self intersect or change
  dimensionality. At each point $x \in \M \subseteq \RR^n$ we have the
  tangent space $T_x\M$ of $\M$ at $x$ which may be seen as a linear
  approximation of $\M$ near $x$. In Fig.~\ref{fig:manifold} a
  2-dimensional manifold embedded in $\RR^3$ is shown together with
  part of a tangent space.

  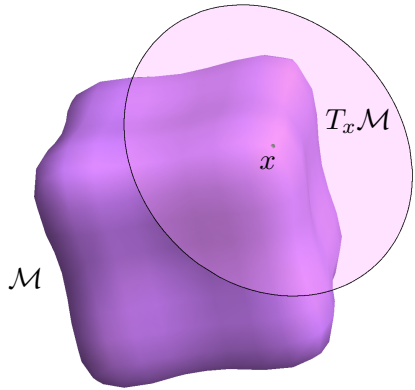
\begin{figure}[ht]
    \centering
    \begin{picture}(150,150)
    \put(0,0){\includegraphics[trim={4.5cm 4cm 1cm
            2.5cm},clip,scale=0.6]{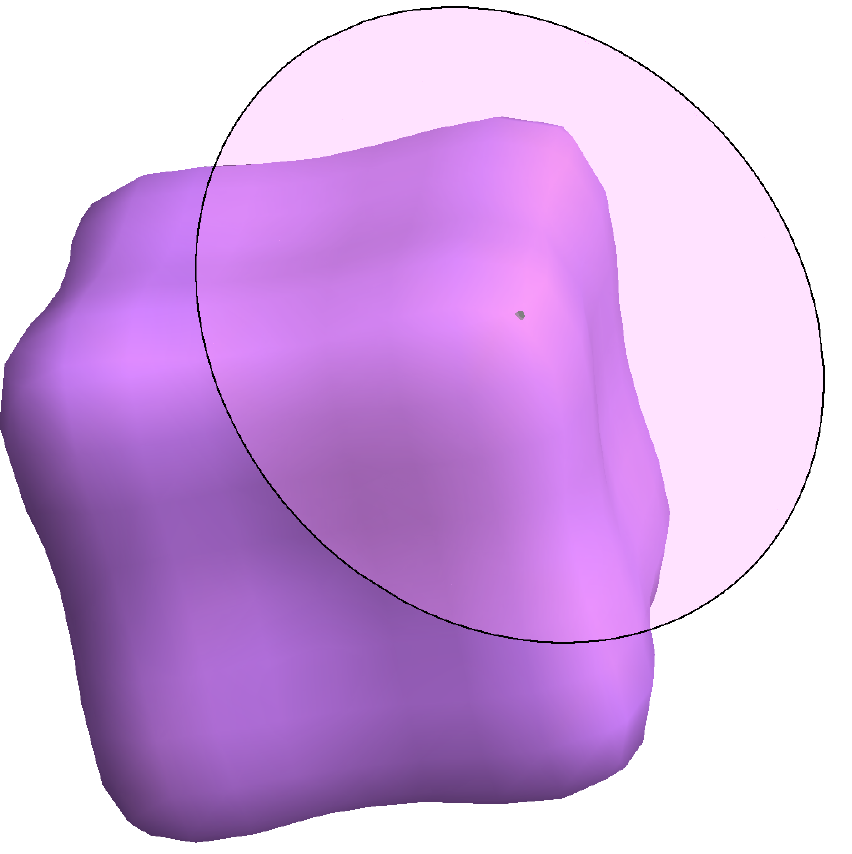}}
      \put(-10,40){$\M$}
      \put(85,85){$x$}
      \put(110,100){$T_x\M$}
    \end{picture}
    \caption{Image of an embedded manifold $\M$ and the tangent space
      $T_x\M$ at a point $x \in \M$.}
    \label{fig:manifold}
  \end{figure}

Let $f: \Z \to \M \subseteq \RR^n$ be a parametrization of an open
subset $f(\Z) \subseteq \M$ defined on some open subset $\Z \subseteq
\RR^d$. A Riemannian metric on $\Z$ is an inner product $\langle \cdot,
\cdot \rangle_z$ on the tangent spaces $T_z\Z \cong T_z\RR^d$ which
varies smoothly from point to point. Here, \emph{smooth} means
infinitely differentiable. Such a metric may be given by a positive
definite $(d\times d)$-matrix $M_z$ which depends smoothly on $z \in
\Z$. The induced inner product is then $\langle v,w \rangle_z=v^TM_zw$
for $v,w \in T_z\RR^d$ seen as column vectors.

Consider the standard inner product between points in $\RR^n$,
$\inner{\x}{\x'} = \sum_{i=1}^n x_i x_i'$ where $x=(x_1,\dots,x_n)$
and $x'=(x_1',\dots,x_n')$. Let $\z \in \Z$ and let $\Delta_1,
\Delta_2 \in U \subset \RR^d$ where $U$ is an open ball centered at
the origin such that $\z + U \subseteq Z$. Then we can compute the
inner product of $\Delta_1$ and $\Delta_2$ at $\z$ using the Taylor
expansion of $f$. Consider the scalar product $\inner{f(\z + \Delta_1)
  - f(\z)}{f(\z + \Delta_2) - f(\z)}$, and the linear part of it in
$\Delta_1$ and $\Delta_2$ which we denote by $\inner{f(\z + \Delta_1)
  - f(\z)}{f(\z + \Delta_2) - f(\z)}_0$.  Then
\begin{align*}
  & \inner{f(\z + \Delta_1) - f(\z)}{f(\z + \Delta_2) - f(\z)}_0= \\
    &\inner{f(\z) + \J_f(z) \Delta_1 - f(\z)}{f(\z) + \J_f(z)\Delta_2
      - f(\z)} = \\ &\inner{\J_f(z) \Delta_1}{\J_f(z)\Delta_2} =
    \Delta_1\T \xtx{\J_f(z)} \Delta_2,
  \end{align*}
where $\J_f(z)$ is the $n \times d$ Jacobian matrix of $f$ at $\z$.
The $d \times d$ symmetric positive definite matrix
$\mat{M}_{\z}=\xtx{\J_f(z)}$ defines a Riemannian metric on $\Z$
induced by $f$ which is called the \emph{pullback metric}. Note that
the pullback metric corresponds to the Riemannian metric on $\M$
induced by the inner product in the ambient space $\RR^n$, which does
not depend on the choice of parametrization $f$. In this way, the
pullback metric avoids the parametrization issue discussed in the
opening section.
  
  \textbf{Distances \& interpolants.} 
  The length of a smooth curve $\vec{c}: [a, b] \rightarrow \Z$
  under the local inner product is
  \begin{align*}
    \mathcal{L}(\vec{c})
       = \int_a^b \sqrt{\dot{\vec{c}}_t\T \mat{M}_{\vec{c}_t} \dot{\vec{c}}_t} \dif{t},
  \end{align*}
  where $\vec{c}_t=c(t)$ and $\dot{\vec{c}}_t = \partial_t \vec{c}(t)$
  is the curve and its derivative, respectively. Natural interpolants
  (geodesics) can then be defined as length minimizing curves
  connecting two points. The length of such a curve is a natural
  distance measure along the manifold.  Unfortunately, minimizing
  curve length gives rise to a poorly determined optimization problem
  as the length of a curve is independent of the parametrization.  The
  following proposition provides remedy \cite{gallot1990riemannian}:
  \begin{proposition}\label{prop:energy}
    Let $\vec{c}: [a, b] \rightarrow \Z$ be a smooth curve that
    (locally) minimizes the curve energy
    \begin{align}
      \mathcal{E}(\vec{c}) &= \frac{1}{2} \int_a^b \dot{\vec{c}}_t\T \mat{M}_{\vec{c}_t} \dot{\vec{c}}_t \dif{t}.
      \label{eq:energy}
    \end{align}
    Then $\vec{c}$ has constant velocity and is locally
    length-minimizing.
  \end{proposition}
  
  \textbf{Integration.}  Given a function $h: \RR^n \rightarrow \RR$ and
  an open subset $\Omega \subseteq \Z$ we can integrate $h$ over
  $f(\Omega)$ as \cite{Pennec:JMIV:06}
  \begin{align*}
    \int_{f(\Omega)} h(\x) \dif{\x}
      &= \int_{\Omega} h(f(\z)) \sqrt{\det(\mat{M}_{\z})} \dif{\z}.
  \end{align*}
  The quantity $\sqrt{\det(\mat{M_z})}$ is known as the
  \emph{Riemannian volume measure} and is akin to the
  Jacobian determinant in the \emph{change of variables theorem}.

  \textbf{Relation to latent variable models.}  As stated in the
  introduction, we are concerned with latent variable models $\x =
  f(\z)$, where $f$ is a smooth stochastic process. The above
  constructions assume that $f$ is deterministic. In this case,
  Riemannian geometry provides us with the tools to make the
  representation \emph{operational} as defined in the
  introduction. This paper is concerned with extending Riemannian
  geometry to the stochastic domain in order to provide operational
  representations in latent variable models. To make the resulting
  constructions practical, we will further show how the stochastic
  geometry can be approximated well by deterministic geometries that
  lend themselves to computations.

\section{Stochastic Riemannian geometry}\label{sec:stoch_riem_geom}

In this section we establish the basic definitions in the area of
random manifolds. To the best of our knowledge this is a fairly
unexplored topic and in our opinion it deserves more
attention. Related work has been done in the area of random fields,
see for instance \cite{AT07}. To illustrate the concepts the
definitions are followed by some elementary examples.

We start by defining random metrics.

\begin{definition}
  Let $\Z \subseteq \RR^d$ be an open subset. A random (or stochastic)
  Riemannian metric is a matrix-valued random field on $\Z$ whose
  sample paths are Riemannian metrics. We also refer to $\Z$ equipped
  with the random metric as a random manifold.
\end{definition}

The stochastic Riemannian metrics considered in this paper are induced
by a stochastic process $f:\Z \to \RR^n$ for some $n$, where $\Z
\subseteq \RR^d$ is an open subset. More precisely, $f$ is a random
field in the sense of \cite{AT07}. In our examples we will have
$\Z=\RR^d$ but it is useful to keep in mind that the parametrization
may not be defined on the whole of $\RR^d$. Now, in order for $f$ to
induce a random metric it must satisfy some differentiability
conditions. One option is to require that any sample path from $f$ is
smooth. In that case we say that $f$ is smooth. If in addition the
sample paths $s:\Z \to \RR^n$ are such that the Jacobian matrix $J_s$
has full rank everywhere, $f$ is called a \emph{stochastic
  immersion}. This implies that $s$ is locally injective and that
$s(\Z) \subseteq \RR^n$ is an immersed submanifold. A stochastic
immersion whose sample paths are injective is called a
\emph{stochastic embedding}. In this case the sample paths $s:\Z \to
\RR^n$ are embedded submanifolds $s(\Z) \subseteq \RR^n$. As such,
$s(\Z)$ has an induced Riemannian metric from the ambient space
$\RR^n$ and this defines a stochastic Riemannian metric on
$\Z$. Another point of view of the basic objects is thus as random
embedded submanifolds of $\RR^n$. These submanifolds all have the same
constant topology and smooth structure induced by $\Z$ and hence from
an intrinsic point of view, the random aspect of these manifolds is
confined to the Riemannian metric. If the conditions to be a
stochastic immersion or embedding are only true with probability 1 we
can modify the process by restricting to the measure 1 subset of the
probability space where the requirements are fulfilled.

As in the deterministic case, the metric on $\Z$ induced by $f$ is
given by the pullback metric
\begin{align*}
  M=J_f^TJ_f,
\end{align*}
where $J_f$ is the Jacobian of $f$. In accordance with the definition
of stochastic Riemannian metrics, $M$ is a matrix-valued stochastic
process parametrized by $\Z$.

If it exists, we may also consider the \emph{expected metric} on
$\Z$ which is given by the mean value $\EE(M)=\EE(J_f^TJ_f)$.
\begin{definition}
  Let $f:\Z \to \RR^n$ be a stochastic immersion such that all the
  entries of the matrix $\EE(M)=\EE(J_f^TJ_f)$ are finite smooth
  functions. We then refer to $\EE(M)$ as the expected metric. It
  defines a Riemannian metric on $\Z$ making it into a Riemannian
  manifold which is called the mean manifold.
\end{definition}
Note that the mean manifold is typically \emph{not} given by the mean value
of $f$, that is the map $\EE(f):\Z \to \RR^n: z \mapsto
\EE(f(z))$. The expected metric has been previously studied for the
\emph{Gaussian process latent variable model} \cite{LHTV14}, and the
\emph{variational autoencoder} \cite{AHH18}.

\begin{remark}
  The requirement that a stochastic process $f:\Z \to \RR^n$ has
  smooth sample paths with probability 1 is not always straightforward
  to verify. An alternative assumption that is relatively easy to
  check is that $f$ is differentiable in mean square, see
  \cite{AT07}. We say that $f$ is mean square smooth if $f$ has mean
  square derivatives of any order. Let $J_f$ denote the mean square
  Jacobian of a mean square smooth process with smooth covariance
  function. The expected metric may be considered in this setting as
  well, assuming that $\EE(M)=\EE(J_f^TJ_f)$ has full rank. The main
  result of this paper, \propref{prop:lengths}, compares expected
  length on random manifolds to length in the expected metric. This
  result holds both in the smooth and mean square smooth setting.
\end{remark}

\begin{example} \label{ex:gaussian}
A significant special case is a Gaussian process or Gaussian random
field
\begin{align*}
  f:\RR^d \to \RR^n: q \mapsto (f_1(q),\dots,f_n(q)),
\end{align*}
with $d \leq n$. Here, the components $f_i:\RR^d \to \RR$ are Gaussian
processes, which means that the vector $(f_i(q_1),\dots,f_i(q_r))$ is
Gaussian for any $q_1,\dots,q_r \in \RR^d$. In fact, that $f:\RR^d \to
\RR^n$ is a Gaussian process means that $\sum_{i=1}^n \alpha_if_i$ is
a Gaussian process for all $\alpha_1,\dots,\alpha_n \in \RR$. We will
be concerned with the case where the random vectors
$(f_i(q_1),\dots,f_i(q_r))$ and $(f_j(q_1),\dots,f_j(q_r))$ are
independent for $i \neq j$ and all $q_1,\dots,q_r \in \RR^d$. The
distribution of $(f_i(q_1),\dots,f_i(q_r))$ is determined by the mean
function $\mu_i:\RR^d \to \RR: p \mapsto \EE(f_i(p))$ and the
covariance function $k_i: \RR^d \times \RR^d \to \RR: (p,q) \mapsto
\EE((f_i(p)-\mu_i(p))(f_i(q)-\mu_i(q)))$. If $k_i$ and $\mu_i$ are
real analytic functions, then $f_i$ is mean square smooth. This can be
seen from the criterion (1.4.9) in \cite{AT07}. Moreover, for any
column vector $v \in \RR^d$,
\begin{align*}
  v^T\EE(J_f^TJ_f)v=\EE(v^TJ_f^TJ_fv)=\EE(||J_fv||^2)=\sum_{i=1}^n\EE(\nabla
  f_i \cdot v)^2 + \sum_{i=1}^n \var{\nabla f_i \cdot v},
\end{align*}
and $\EE(\nabla f_i \cdot v)=\nabla \mu_i \cdot v$. Hence $\EE(J_f^TJ_f)$
has full rank for example if $\mu=(\mu_1,\dots,\mu_n)$ is locally
injective or $\nabla f_i$ has non-degenerate covariance matrix for
some $i$.

For Gaussian processes in machine learning, see \cite{L05, RW06,
  LHTV14} and in particular the Gaussian process latent variable
model.
\end{example}

\begin{example} \label{ex:randomproj}
Another special case to keep in mind which is particularly simple is a
stochastic embedding $f(x)=L(g(x))$, where $g:\RR^d \to \RR^m$ a
deterministic smooth injection and $L:\RR^m \to \RR^n$ is a random
matrix with $n \leq m$. Consider the case where $L$ is given by a
distribution on $\SO{\RR^m}{\RR^n}$, that is any sample of $L$ has
orthonormal rows. Let $W$ be the span of the rows of $L$ and fix
coordinates on $W$ given by the rows of $L$ as a basis. In this way,
$f$ can be seen as a random projection of the deterministic manifold
$g(\RR^d) \subseteq \RR^m$ into $W$. Note that if $n<m$, this can be
viewed as random dimensionality reduction. See \cite{B06} for a survey
on the topic of random projections in machine learning. In this
context it is relevant to recall the Johnson-Lindenstrauss
lemma. Roughly speaking, given a finite set $E \subset \RR^m$,
distances in $E$ are well preserved by a random projection $\RR^m \to
\RR^n$ if $n$ is big enough. See for example \cite{DG03} for
background on the Johnson-Lindenstrauss lemma. In the same spirit one
can show that distances along the submanifold $g(\RR^d) \subseteq
\RR^m$ are well preserved under similar circumstances \cite{BW09,
  C08, DFMV08, BHW08, GGL16, V11}.
\end{example}

\begin{example} \label{ex:simple}
Consider a simple random manifold given by a stochastic embedding
\begin{align*}
  f:\RR^d \to \RR^n:z
  \mapsto (f_1(z),\dots,f_n(z)),
\end{align*}
where $f_i=\mu_i + \epsilon_i\sigma_i$ for smooth functions
$\mu_i,\sigma_i:\RR^d \to \RR$ and random variables $\epsilon_i$ with
$\EE(\epsilon_i)=0$ and $\EE(\epsilon_i^2)=1$. Let $\mu:\RR^d \to
\RR^n:z \mapsto (\mu_1(z),\dots,\mu_n(z))$, $\sigma:\RR^d \to
\RR^n:z\mapsto (\sigma_1(z),\dots,\sigma_n(z))$ and define
\begin{align*}
e=
\begin{pmatrix}
  \epsilon_1 & 0 & 0 & \ldots & 0\\
  0 & \epsilon_2 & 0 & \ldots & 0\\
  &&\vdots&\\
  0 & 0 & 0 & \ldots & \epsilon_n
  \end{pmatrix}.
\end{align*}
Then $f=\mu + e \sigma$, where the product $e\sigma$ is matrix
multiplication and $\mu$ and $\sigma$ are viewed as column
vectors. Note that $\EE(e)=0$ and $\EE(e^2)=\id{\RR^n}$, that is
$\EE(e)$ is the zero matrix and $\EE(e^2)$ the identity matrix on
$\RR^n$. For the Jacobian of $f$ we get $J_f=J_{\mu}+eJ_{\sigma}$ and
hence
\begin{align*}
M=J_f^TJ_f =
(J_{\mu}+eJ_{\sigma})^T(J_{\mu}+eJ_{\sigma})=J_{\mu}^TJ_{\mu}+J_{\mu}^TeJ_{\sigma}+J_{\sigma}^TeJ_{\mu}+J_{\sigma}^Te^2J_{\sigma}.
\end{align*}
Therefore
$\EE(J_f^TJ_f)=J_{\mu}^TJ_{\mu}+J_{\mu}^T\EE(e)J_{\sigma}+J_{\sigma}^T\EE(e)J_{\mu}+J_{\sigma}^T\EE(e^2)J_{\sigma}=J_{\mu}^TJ_{\mu}+J_{\sigma}^TJ_{\sigma}$. This
is also the metric on $\RR^d$ induced by the embedding $g:\RR^d \to
\RR^{2n}=\RR^n \times \RR^n$ given by 
\begin{align*}
  g:\RR^d \to \RR^n \times
  \RR^n:z \mapsto (\mu(z),\sigma(z)).
\end{align*}
The metric induced by $g$ has been studied
empirically for variational autoencoders \cite{AHH18}. We see in this
example how the mean manifold $g(\RR^d) \subseteq \RR^{2n}$ depends on
both the mean value $\mu=\EE(f)$ and the standard deviation map
$\sigma$.

Another point of view is that $f = L \circ g$ where
\begin{align*}
  L=(\id{\RR^n} |\; e)
\end{align*}
is the random $(n \times 2n)$-matrix given by stacking the columns of
the matrices $\id{\RR^n}$ and $e$ side by side. This is a special case
of the random manifolds considered in \exref{ex:randomproj}. In
connection with Ex.~\ref{ex:decomp} below we note that
$\EE(L^TL)=\id{\RR^{2n}}$.
\end{example}

\begin{example} \label{ex:decomp}
Recall \exref{ex:randomproj} and random manifolds given by $f=L \circ
g$ where $g: \RR^d \to \RR^m$ is an embedding and $L:\RR^m \to \RR^n$
is a random matrix. Assume that $\EE(L^TL)$ has finite entries and
full rank. We shall see that, similarly to \exref{ex:simple}, the mean
manifold is essentially the manifold $g(\RR^d) \subseteq \RR^m$. Since
$\EE(L^TL)$ is symmetric we may diagonalize $\EE(L^TL)$ with an
orthogonal $(m\times m)$-matrix $P$, that is $\EE(L^TL)=P^TDP$ where $D$
is diagonal. Since $\EE(L^TL)$ is positive definite we can take the
square roots of the eigenvalues on the diagonal of $D$ to obtain a
diagonal matrix $\sqrt{D}$. In this case the stochastic metric is
$M=J_f^TJ_f=J_g^TL^TLJ_g$ and the mean metric is
$\EE(M)=J_g^T\EE(L^TL)J_g$. Hence the mean manifold is induced by the
embedding $\RR^d \to \RR^m:z\mapsto \sqrt{D}Pg(z)$, which is just the
manifold $g(\RR^d) \subseteq \RR^m$ up to a change of coordinates. In
contrast consider $\EE(f):\RR^d \to \RR^n$, and note that
$\EE(f)=\EE(L) \circ g$. The image of $\EE(f)$ is thus the image of
the embedded mean manifold $\sqrt{D}Pg(\RR^d) \subseteq \RR^m$ under
the linear map $\EE(L)(\sqrt{D}P)^{-1}:\RR^m \to \RR^n$.
\end{example}

\subsection{Expected length versus the expected metric} \label{sec:versus}
Let $\Z \subseteq \RR^d$ be an open subset. We will now explore the
topic of shortest paths on a random manifold induced by a stochastic
immersion $f:\Z \to \RR^n$ with expected metric $\EE(M)$. In order to
talk about expected length of curves on the random manifold we need
that $f$ is measurable when viewed as a map $\Omega \times Z \to
\RR^n$, where $\Omega$ is a probability space. Here, $Z$ and $\RR^n$
are endowed with the Lebesgue measure and $\Omega \times Z$ with the
product measure. Now let $\vec{c}: [a, b] \rightarrow \Z$ denote a
smooth immersed curve and consider its stochastic immersion $f \circ
c$ in $\RR^n$. We stress that $\vec{c}$ is a deterministic curve in
$\Z$, while $f \circ c $ is a random curve in $\RR^n$. The energy of
$c$, defined as in Eq.~\ref{eq:energy}, is a random quantity and it is
natural to consider its expectation with respect to the random metric.
Since the energy integrand is positive, Tonelli's Theorem tells us
that the expected energy $\epsilon(\vec{c})$ is given by
  \begin{align*}
    \epsilon(\vec{c})
    &= \E(\mathcal{E}(c))
       = \frac{1}{2} \E\left(\int_a^b \dot{\vec{c}}_t\T \mat{M}_{\vec{c}_t} \dot{\vec{c}}_t \dif{t}\right) \\ 
      &= \frac{1}{2} \int_a^b \dot{\vec{c}}_t\T \E(\mat{M}_{\vec{c}_t}) \dot{\vec{c}}_t \dif{t}.
  \end{align*}
  This implies that a curve $\vec{c}$ with minimal expected energy
  over the stochastic manifold is a geodesic under the deterministic
  Riemannian metric $\E( \mat{M} )$.
  
  We can understand a curve with minimal expected energy in more
  explicit terms as follows. Let $u_t = \E(\| \dot{\vec{c}}_t \|)$ and
  $v_t = 1$ denote two functions over the interval $[a, b]$; here we
  use the shorthand notation $\| \dot{\vec{c}}_t \| =
  \sqrt{\dot{\vec{c}}_t\T \mat{M}_{\vec{c}_t} \dot{\vec{c}}_t}$.  The
  Cauchy-Schwartz inequality then gives
  \begin{align*}
    \left( \int_a^b \E(\| \dot{\vec{c}}_t \|) \dif{t} \right)^2
      &\leq  \int_a^b \E(\| \dot{\vec{c}}_t \|)^2 \dif{t} \cdot \int_a^b \dif{t} \\
      &= (b-a) \int_a^b \E(\| \dot{\vec{c}}_t \|)^2 \dif{t}.
  \end{align*}
  Let $l(\vec{c})\! = \! \E(\mathcal{L}(\vec{c}))$ denote the expected
  length of $\vec{c}$. Then
  \begin{align} \label{eq:CS}
    \int_a^b \E(\| \dot{\vec{c}}_t \|)^2 \dif{t}
      &\geq 
    \frac{l^2(\vec{c})}{b-a}.
  \end{align}
  Equality is achieved when $u_t$ and $v_t$ are parallel, that is when
  $\E(\| \dot{\vec{c}}_t \|)$ is constant. If the curve is
  \emph{regular} in the sense that the expected speed $\E(\|
  \dot{\vec{c}}_t \|)$ is non-zero for all $t$, then we can always
  reparametrize $\vec{c}_t$ to have constant expected speed and
  achieve equality in Eq.~\ref{eq:CS}. Since $\var{\| \dot{\vec{c}}_t
    \|} = \E(\| \dot{\vec{c}}_t \|^2) - \E(\| \dot{\vec{c}}_t \|)^2$,
  we see that
  \begin{align*}
    \int_a^b \!\E(\| \dot{\vec{c}}_t \|)^2 \dif{t}
      &= \int_a^b \! \E(\| \dot{\vec{c}}_t \|^2) \dif{t}
       - \int_a^b \var{\| \dot{\vec{c}}_t \|} \dif{t} \\
      &= 2\epsilon(\vec{c})
       - \int_a^b \var{\| \dot{\vec{c}}_t \|} \dif{t}.
  \end{align*}
  Assuming that the curve has constant expected speed, we then get
  \begin{align*}
    \epsilon(\vec{c})
      &= 
    \frac{l^2(\vec{c})}{2(b-a)} +
    \frac{1}{2}\int_a^b \var{\| \dot{\vec{c}}_t \|} \dif{t}.
  \end{align*}
  Minimizing expected curve energy, thus, does not always minimize the expected
  curve length. Rather, this balances the minimization of expected curve length
  and the minimization of curve variance.
  
\section{Expected length in high codimension} \label{sec:rates}

Let $\Z \subseteq \RR^d$ be an open subset and $f: \Z \to \RR^n$ a
stochastic immersion with expected metric $\EE(M)$. Any smooth curve
$c:[0,1] \to \Z$ gives rise to a random curve $f \circ c:[0,1] \to
\RR^n$. In this section we continue to examine the relationship
between the expected length of random curves and the length with
respect to the expected metric. More precisely, we show that length
with respect to the expected metric is a good approximation for
expected length in high ambient dimension $n$. Assuming independence
of the components of $f \circ c$ we could apply a version of the
central limit theorem such as the Berry-Esseen theorem to this
problem, see for example \cite{E69}. We found it more convenient to
take a direct approach via the Taylor expansion of the norm of
velocity vectors. See also \cite{K04, BM13} and Chapter 27 of
\cite{C62} for approximation results in the same vein.

\subsection{Expected norm of high dimensional vectors}

We will consider a sequence $W_n$ of random vectors in $\RR^n$. This
means that for each integer $1 \leq n$ we have an $\RR^n$-valued
random variable $W_n$. Consider the norm $w_n = ||W_n||$. Assume that
$w_n^2$ has first and second moments and put $m_n = \sqrt{\EE(w_n^2)}$
and $\Sigma_n = \sqrt{n\cdot \var{w_n^2}}$. We say that $m_n$ is
bounded away from 0 if there is a constant $0<b$ such that $b<m_n$ for
all $n$. For $1 \leq k$, let $\mu_k(w_n^2)=\EE((w_n^2-m_n^2)^k)$
denote the $k$-th central moment of $w_n^2$.

\begin{definition}
  We call a sequence $W_n$ as above balanced if $m_n$ is bounded away
  from 0 and $m_n$, $n^2 \mu_3(w_n^2)$ and $n^2 \mu_4(w_n^2)$ are
  bounded sequences.
\end{definition}

Suppose that $W_n$ is balanced. We shall see that $\Sigma_n$ is a
bounded sequence. Let $Z_n=n(w_n^2-m_n^2)^2$ and let $\mathcal{X}_n$
be the indicator function for the event $\{Z_n < 1\}$. We know that
$\EE(Z_n^2)=n^2\mu_4(w_n^2)$ is bounded and need to show that
$n\mu_2(w_n^2)=\EE(Z_n)$ is bounded. This follows from $\EE(Z_n) \leq
1+\EE(Z_n(1-\mathcal{X}_n)) \leq 1 + \EE(Z_n^2)$.

\begin{remark} \label{rem:sqrt}
Consider the cubic Taylor polynomial of $\sqrt{x}$ around $x=1$,
$P(x)=1+\frac{1}{2}(x-1)-\frac{1}{8}(x-1)^2+\frac{1}{16}(x-1)^3$. We
will show that $|\sqrt{x}-P(x)| \leq \frac{5}{16}(x-1)^4$ for all $x
\geq 0$. Let $y=\sqrt{x}$ and put $Q(y)=P(y^2)-y$. Since
$16Q(y)=(y-1)^4(y^2+4y+5)$, we have that $Q(y) \geq 0$ for $y \geq
0$. Hence $P(x) - \sqrt{x} \geq 0$ for $x \geq 0$. Let
$R(y)=\frac{5}{16}(y^2-1)^4-P(y^2)+y$. Since
$16R(y)=y(y-1)^4(5y^3+20y^2+29y+16) \geq 0$ for $y \geq 0$, we have
that $\frac{5}{16}(x-1)^4 \geq P(x) - \sqrt{x}$ for $x \geq 0$.
\end{remark}

\begin{proposition} \label{prop:estimate}
Let $W_n$ be a balanced sequence of vectors. Then
\begin{align*}
  \EE(w_n) = m_n - \frac{\Sigma_n^2}{8n\cdot m_n^3} +
  \mathcal{O}(n^{-2}).
\end{align*}
\end{proposition}

\begin{proof}
We will first prove the statement assuming that $m_n=1$ for all
$n$. In this case we need to show that
$\EE(w_n)-(1-\Sigma_n^2/8n)=\mathcal{O}(n^{-2})$. Let
$P(x)=1+\frac{1}{2}(x-1)-\frac{1}{8}(x-1)^2+\frac{1}{16}(x-1)^3$ be
the cubic Taylor polynomial of $\sqrt{x}$ around $x=1$ and note that
$\EE(P(w_n^2))=1-\Sigma_n^2/8n+\frac{1}{16}\mu_3(w_n^2)$. Since by
assumption $\mu_3(w_n^2) = \mathcal{O}(n^{-2})$, it is enough to show
that $\EE(w_n)-\EE(P(w_n^2))=\mathcal{O}(n^{-2})$. Note that
\begin{align*}
  |\EE(w_n)-\EE(P(w_n^2))| =|\EE(w_n-P(w_n^2))| \leq \EE(|w_n-P(w_n^2)|).
\end{align*}
Also, $|\sqrt{x}-P(x)| \leq \frac{5}{16}(x-1)^4$ for all $x \geq 0$ by
\remref{rem:sqrt}. Hence $|w_n-P(w_n^2)| \leq (w_n^2-1)^4$ and by
assumption we have $\EE((w_n^2-1)^4)=\mathcal{O}(n^{-2})$. It follows
that $\EE(|w_n-P(w_n^2)|) \leq \EE((w_n^2-1)^4)= \mathcal{O}(n^{-2})$.

Now consider the general case of a balanced sequence $W_n$ and let
$m_n$ and $\Sigma_n$ denote the corresponding sequences associated to
$W_n$. Since $W_n/m_n$ is balanced, $\EE(||W_n/m_n||^2)=1$ for all $n$
and $n \cdot \var{||W_n/m_n||^2}=\Sigma_n^2/m_n^4$, we have by above
that
\begin{align*}
  \EE(w_n/m_n)=1-\frac{\Sigma_n^2/m_n^4}{8n}+\mathcal{O}(n^{-2}).
\end{align*}
But $m_n$ is bounded and so the claim follows by multiplying by $m_n$.
\end{proof}

\begin{remark} \label{rem:estimate}
  Suppose that $W_n$ is balanced. Since $\Sigma_n$ is bounded and
  $m_n$ bounded away from 0 we have that $\limsup_{n \to \infty}
  \Sigma_n < \infty$ and $\liminf_{n \to \infty} m_n > 0$. Let $A,b
  \in \RR$ be such that $A > \limsup_{n \to \infty} \Sigma_n$ and $0 <
  b < \liminf_{n \to \infty} m_n$. Then, by \propref{prop:estimate},
  \begin{equation} \label{eq:estimate}
  0 \leq m_n - \EE(w_n) \leq
  \frac{A^2}{8n \cdot b^3},
\end{equation}
  for large enough $n$. In particular, if $\Sigma_n \to \Sigma$ and
  $m_n \to m$ with $\Sigma, m \in \RR$, then Eq.~\ref{eq:estimate}
  holds for any $A > \Sigma$ and $0 < b < m$. In this case we also
  have that
  \begin{align} \label{eq:limit}
    \lim_{n \to \infty} \EE(w_n)=\lim_{n
    \to \infty} m_n
  \end{align}
  and \propref{prop:estimate} gives some
  additional information concerning the rate of convergence. Note also
  that $\var{w_n}=m_n^2-\EE(w_n)^2 \to 0$ for balanced $W_n$.
\end{remark}

\subsection{Normed sequences of independent random variables}

For $k\geq 1$ and a random variable $Y$ with first moment $\EE(Y)$,
let $\mu_k(Y)=\EE((Y-\EE(Y))^k)$ denote the $k$-th central moment. Now
consider a sequence $X_1,X_2,\dots$ of independent random variables
and let \[W_n = (X_1/\sqrt{n},\dots,X_n/\sqrt{n})\] be the
corresponding normalized sequence of vectors. We say that the sequence
$X_1,X_2,\dots$ has \emph{bounded moments} if the moments $\EE(X_i^k)$
form a bounded sequence for any $k \geq 1$. This implies that the
central moments $\mu_k(X_i)$ are bounded as well. Let
$w_n=||W_n||$. The definition of a balanced sequence of vectors $W_n$
is motivated by this setup since, for instance, in this case $m_n^2 =
\EE(w_n^2)=\frac{1}{n}\sum_{i=1}^n \EE(X_i^2)$ is bounded if the
sequence $X_1,X_2,\dots $ has bounded moments. Similarly,
$\Sigma_n^2=n\cdot \var{w_n^2}=n\cdot
\frac{1}{n^2}\sum_{i=1}^n\var{X_i^2}$ is bounded in this case.

\begin{lemma} \label{lemma:balanced}
Let $X_1,X_2,\dots$ be independent with bounded moments. If $m_n$ is
bounded away from 0 then $W_n=(X_1/\sqrt{n},\dots,X_n/\sqrt{n})$ is
balanced.
\end{lemma}
\begin{proof}
We need to show that $n^2\mu_3(w_n^2)$ and $n^2\mu_4(w_n^2)$ are
bounded. Note that
$n^2\mu_3(w_n^2)=n^2\mu_3(\frac{1}{n}\sum_{i=1}^nX_i^2)=\frac{1}{n}\sum_{i=1}^n\mu_3(X_i^2)$,
which is bounded. Similarly,
\[n^2\mu_4(w_n^2)=\frac{1}{n^2}\mu_4(\sum_{i=1}^nX_i^2)=\frac{1}{n^2}\sum_{i=1}^n\mu_4(X_i^2)+\frac{6}{n^2}\sum_{i<j}\mu_2(X_i^2)\mu_2(X_j^2),\] which is bounded as well.
\end{proof}

\subsection{Expected length of curves} \label{sec:curves}

Consider a sequence of stochastic processes $f_1,f_2,\dots$ defined on
$[0,1]$ such that for any $t \in [0,1]$, $f'_1(t),f'_2(t),\dots$ are
independent. Let
\begin{align*}
  \phi_n:[0,1] \to \RR^n:t \mapsto
  (f_1(t)/\sqrt{n},\dots,f_n(t)/\sqrt{n})
\end{align*}
and assume that $\phi_n$ is a stochastic immersion or a mean square
smooth process on $(0,1)$. As in \secref{sec:versus}, we also assume
that $\phi_n$ has an expected metric and is measurable when seen as a
map $\Omega \times [0,1] \to \RR^n$, where $\Omega$ is a probability
space. Furthermore, suppose that the sequence $f_1',f_2',\dots$ has
\emph{uniformly bounded moments} in the sense that for any $k \geq 1$,
there is a constant $C_k$ such that $|\EE(f_i'(t)^k)| \leq C_k$ for
all $i$ and $t \in [0,1]$. Let $w_n(t) = ||\phi'_n(t)||$, $m_n(t) =
\sqrt{\EE(w_n^2(t))}$ and $\Sigma_n(t) = \sqrt{n\cdot
  \var{w_n^2(t)}}$. Then $\sup_{n,t} \Sigma_n(t) <
\infty$. Furthermore, we assume that $m_n$ is uniformly bounded away
from 0, meaning that $0 < \inf_{n,t} m_n(t)$. Let $L_n$ denote the
length of $\phi_n$ in the expected metric and $l_n$ the expected
length of $\phi_n$. In other words
\begin{equation} \label{eq:lengths}
L_n =
\int_0^1 \EE(||\phi_n'(t)||^2)^{1/2} \;\dif{t}, \quad l_n = \EE(\int_0^1
||\phi_n'(t)|| \;\dif{t}) = \int_0^1 \EE(||\phi_n'(t)||) \;\dif{t}.
\end{equation}
Let $\sup_{n,t} \Sigma_n(t) < A$ and $0 < b < \inf_{n,t} m_n(t)$. By
\lemmaref{lemma:balanced} and Eq.~\ref{eq:estimate}, for any $t \in [0,1]$
there is a $N_t >0 $ such that $m_n(t) - \EE(w_n(t)) \leq A^2/(8n
\cdot b^3)$ for all $n > N_t$. In fact, due to the uniform bounds on
the moments $\EE(f_i'(t)^k)$ we have that for large enough $n$,
$m_n(t) - \EE(w_n(t)) \leq A^2/(8n \cdot b^3)$ for all $t \in [0,1]$.

\begin{proposition} \label{prop:lengths}
With $\sup_{n,t} \Sigma_n(t) < A$ and $0 < b < \inf_{n,t} m_n(t)$, 
\begin{align*}
   0
  \leq \frac{L_n-l_n}{L_n} \leq \frac{A^2}{8nb^4},
\end{align*}
for large enough $n$.
\end{proposition}
\begin{proof}
Since for large enough $n$, $0 \leq m_n(t) - \EE(w_n(t)) \leq A^2/(8n
\cdot b^3)$ for all $t \in [0,1]$, integrating both sides over $[0,1]$
gives $0\leq L_n-l_n\leq A^2/(8nb^3)$. Divide by $b$ and note that $b
< L_n$ for all $n$.
\end{proof}

This result implies that in high codimension we can minimize expected
energy instead of minimizing expected length. A curve minimizing
expected energy can be found by computing the expected metric and
using standard tools from differential geometry to recover a geodesic
associated with this metric. It it interesting to note that the length
in the expected metric bounds the expected length from
above. Consequently, by minimizing expected energy we minimize an
upper bound on expected length. Such notions are standard in
\emph{variational inference} \cite{blei2017variational}.

\begin{remark}
Expected speed and expected length of random curves are quite natural
quantities to consider. For instance, minimal expected length is an
interesting candidate for a distance measure along random
manifolds. What about simply taking the expectation value $\EE(\phi_n)$
component wise and considering the length of this curve? This seems
natural enough but is not enough to capture the notion of expected
length. The velocity of the curve $\EE(\phi_n)$ at $t \in [0,1]$ is
$||\EE(\phi'_n(t))||$. Let $t \in [0,1]$ and assume for simplicity that
$\EE(||\phi'_n(t)||)$ converges as $n \to \infty$. By
\remref{rem:estimate}, $\EE(||\phi'_n(t)||^2)$ converges as well and
\begin{align*}
\lim_{n \to \infty} \EE(||\phi'_n(t)||^2)= \lim_{n \to \infty} \EE(||\phi'_n(t)||)^2.
\end{align*}
Also, $\EE(||\phi'_n(t)||^2)=\frac{1}{n}\sum_i \EE(f'^2_i(t))=\frac{1}{n}\sum_i
\var{f'_i(t)} + ||\EE(\phi'_n(t))||^2$.  Thus,
\begin{align*}
  \lim_{n \to \infty} ||\EE(\phi'_n(t))|| \neq \lim_{n \to \infty}
  \EE(||\phi'_n(t)||)
\end{align*}
unless $\frac{1}{n}\sum_i \var{f'_i(t)} \to 0$ as $n \to \infty$.
\end{remark}

\section{Gaussian processes}\label{sec:empirical}

We will now have a closer look at the case of Gaussian processes
\cite{RW06}.

\subsection{Definitions}

For a smooth function $h:\RR^d \to \RR:(p_1,\dots,p_d) \mapsto
h(p_1,\dots,p_d)$ we will use the notation
\begin{align*}
  h_{p_{i_1},\dots,p_{i_j}}=\frac{\partial^jh}{\partial p_{i_1} \dots \partial p_{i_j}},
\end{align*}
where $i_1,\dots,i_j \in \{1,\dots,d\}$. A Gaussian process $f:\RR^d
\to \RR$ is a stochastic process such that for any $q_1,\dots,q_r \in
\RR^d$, $(f(q_1),\dots,f(q_r))$ is a Gaussian vector.  The
distribution of the vector $(f(q_1),\dots,f(q_r))$ is determined by
the mean function $\mu=\EE(f):\RR^d \to \RR:p \mapsto \EE(f(p))$ and
covariance function 
\begin{align*}
  k:\RR^d \times \RR^d \to \RR:(p,q) \mapsto
  \covar{f(p)}{f(q)}.
\end{align*}
The function $k$ is also known as the kernel of
$f$. We will assume that $\mu$ and $k$ are real analytic functions.
As explained in \exref{ex:gaussian}, this implies that $f$ is mean
square smooth. For such a Gaussian process $f: \RR^d \to \RR$, the
partial derivative $f_{p_i}$ with respect to $p_i$ for $i \in
\{1_,\dots,d\}$ is a Gaussian process with mean function $\mu_{p_i}$
and kernel $k_{p_i,q_i}$. More generally, for $p,q \in \RR^d$,
$f_{p_i}(p)$ and $f_{p_j}(q)$ have covariance
\begin{align*}
  \covar{f_{p_i}(p)}{f_{p_j}(q)}=k_{p_i, q_j}(p,q),
\end{align*}
see \cite{RW06}. Gaussian processes $f_1,\dots,f_m:\RR^d \to \RR$ are
\emph{independent} if \[(f_1(q_1),\dots,f_1(q_r)), \dots,
(f_m(q_1),\dots,f_m(q_r))\] are independent for all $q_1,\dots,q_r \in
\RR^d$. A Gaussian process $f:\RR^d \to \RR^m: p \mapsto
(f_1(p),\dots,f_m(p))$ is a stochastic process such that $\sum_{i=1}^m
\alpha_i f_i$ is a Gaussian process for all $\alpha_1,\dots,\alpha_m
\in \RR$. We will call such a process \emph{symmetric} if
$f_1,\dots,f_m$ are independent and all have the same kernel.

\subsection{Gaussian process latent variable models}

Gaussian processes are used in machine learning in the context of
Gaussian process latent variable models (GPLVMs), see \cite{L05, RW06,
  LHTV14}. In GPLVMs we consider a Gaussian process prior $F:\RR^d \to
\RR^m$ which is symmetric, has zero mean and whose kernel is of a
particular form, depending on finitely many hyper parameters. A common
choice is a Radial Basis Function (RBF) kernel
\begin{align*}
  K(p,q)=\sigma_0^2 e^{-\frac{1}{2l^2}||p-q||^2}
\end{align*}
with variance $\sigma_0^2$ and length scale $l > 0$. Another example
is the linear covariance case where the kernel is given by the
Euclidean scalar product: $K(p,q)=p^Tq$ where $p, q \in \RR^d$ are
seen as column vectors. In this case, the GPLVM reduces to
probabilistic principal component analysis \cite{L05}. Typically, the
data is assumed to be observed with additive iid Gaussian noise with
variance $\sigma_1^2$, which introduces one more hyper parameter for
the model. Given a finite set of data points $Y \subset \RR^m$ we
solve a maximum likelihood type problem to compute the hyper
parameters of the model as well as a set of latent points $X \subseteq
\RR^d$ corresponding to the data. This can be done for example using
the software package \cite{GPy}. Combined with Gaussian process
regression, the result is a Gaussian process posterior $f:\RR^d \to
\RR^m:p \mapsto f(p)=(f_1(p),\dots,f_m(p))$ which fits the data. The
process $f$ is a symmetric Gaussian process with mean and kernel given
below.

For matrices $A \in \RR^{d \times r}$ and $B \in \RR^{d \times s}$
with columns $A=(a_1,\dots,a_r)$ and $B=(b_1,\dots,b_s)$ we will use
$K(A,B)$ to denote the matrix given by
$\{K(A,B)\}_{i,j}=K(a_i,b_j)$. Assume that the data points are
distinct, let $N=|X|=|Y|$ be the number of data points and consider $Y
\in \RR^{m\times N}$ and $X \in \RR^{d\times N}$. Let
$R(X)=(K(X,X)+\sigma_1^2 \id{\RR^N})$ and assume that $R(X)$ is
invertible. This is the case if the kernel $K$ is positive definite in
the sense that $K(X,X)$ is positive definite, or $K$ is semi-positive
definite and $\sigma_1 \neq 0$. The mean $\mu$ and kernel $k$ of the
posterior process are then given by $\mu(p)=YR(X)^{-T}K(X,p)$ and
$k(p,q) = K(p,q) - K(p,X)R(X)^{-1}K(X,q)$, see \cite{RW06} for
details.

Let $f:\RR^d \to \RR^m$ be the posterior of a GPLVM. If the kernel of
the prior is real analytic then so is the kernel and mean of $f$ and
hence $f$ is mean square smooth. Let $J_f$ denote the mean square
Jacobian. For $p \in \RR^d$, the metric induced by $f$ at $p$ is
$J_f(p)^TJ_f(p)$, which follows a non-central Wishart distribution,
see \cite{M05}.

\begin{example}
Consider a GPLVM $f:\RR^d \to \RR^m$ with prior kernel $K(p,q)=p^Tq$
where $p,q \in \RR^d$ are column vectors. As we shall see, the
expected metric is a constant matrix in this case. This means that the
mean manifold is flat and that geodesics are straight lines for this
choice of kernel for the prior. Let $\mu=\EE(f):\RR^d \to \RR^m$ and
note that $\mu(p)=YR(X)^{-T}X^Tp$ with notation as above. Hence
$\EE(J_f)=\J_{\mu}=YR(X)^{-T}X^T$ is constant. Differentiating the
posterior kernel we get that the expected metric is given by
\[
\EE(J_f^TJ_f)=\EE(J_f)^T\EE(J_f)+m\cdot(\id{d}-XR(X)^{-1}X^T).
\]

\end{example}

\subsubsection{Empirical illustration}

Consider the posterior process $f:\RR^d \to \RR^m$ of a GPLVM with
$f(p)=(f_1(p),\dots,f_m(p))$ and its projections $\Phi_n:\RR^d \to
\RR^n:p \mapsto (f_1(p),\dots,f_n(p))$ for $n \leq m$. Given a curve
$c:[0,1] \to \RR^d$ we acquire a sequence of Gaussian curves $\phi_n =
\Phi_n \circ c:[0,1] \to \RR^n$ up to $n=m$.

\begin{example}
To illustrate the results of \secref{sec:curves}, consider the
120$\times$120-pixel image of a bird in
\figref{fig:bird_rot-a}. Images of this resolution may be seen as
points in $\RR^m$, where $m=120^2=14400$. We produce a sequence of $N$
points in $\RR^m$ by rotating the image (using interpolation) by an
angle $2\pi k/N$ for $k=0,\dots,N-1$, see \figref{fig:bird_rot}.
\begin{figure}[ht]
\centering
\begin{subfigure}[]{.13\linewidth}
\includegraphics[scale=0.5]{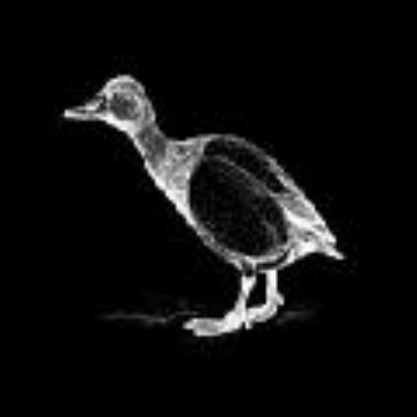}
\caption{}\label{fig:bird_rot-a}
\end{subfigure}
\begin{subfigure}[]{.13\linewidth}
\includegraphics[scale=0.5]{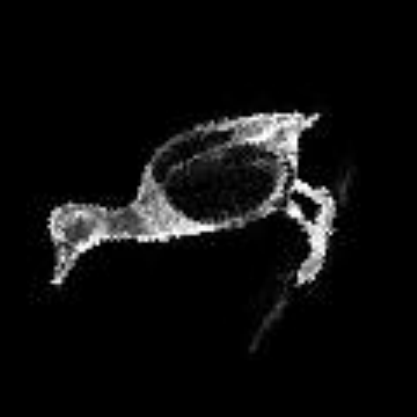}
\caption{}\label{fig:bird_rot-b}
\end{subfigure}
\begin{subfigure}[]{.13\linewidth}
\includegraphics[scale=0.5]{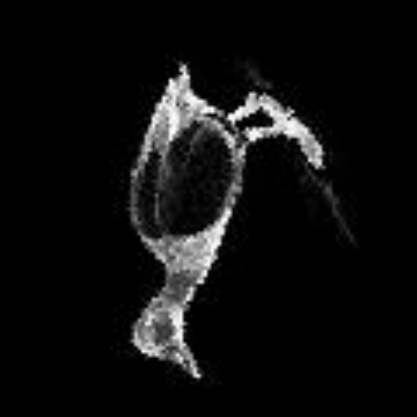}
\caption{}\label{fig:bird_rot-c}
\end{subfigure}
\begin{subfigure}[]{.13\linewidth}
\includegraphics[scale=0.5]{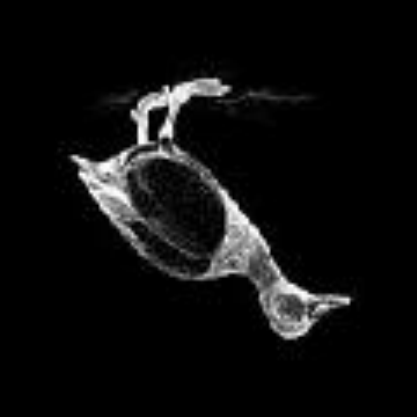}
\caption{}\label{fig:bird_rot-d}
\end{subfigure}
\begin{subfigure}[]{.13\linewidth}
\includegraphics[scale=0.5]{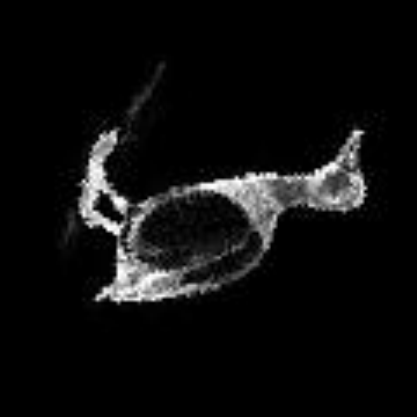}
\caption{}\label{fig:bird_rot-e}
\end{subfigure}
\begin{subfigure}[]{.13\linewidth}
\includegraphics[scale=0.5]{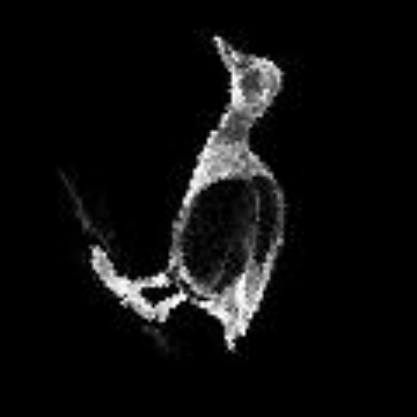}
\caption{}\label{fig:bird_rot-f}
\end{subfigure}
\begin{subfigure}[]{.13\linewidth}
\includegraphics[scale=0.5]{./bird_full0.pdf}
\caption{}\label{fig:bird_rot-g}
\end{subfigure}
\caption{A rotated image.}
\label{fig:bird_rot}
\end{figure}

Feeding these to a GPLVM with RBF kernel and $d$-dimensional latent
space we obtain a Gaussian process $f:\RR^d \to \RR^m$ together with a
sequence of $N$ latent points $X \subset \RR^d$. For any $n \leq m$ we
have a Gaussian process $\Phi_n:\RR^d \to \RR^n$ given by projection onto
the $n$ first coordinates.

Let $c:[0,1] \to \RR^d$ be the line segment joining the first two
points of $X$ and put $\phi_n = \Phi_n \circ c$. Let $l_n$ and $L_n$
be given as in Eq.~\ref{eq:lengths}. Using \cite{GPy} with $d=6$ and
$N=100$ we have estimated $l_n$ and $L_n$ empirically by sampling the
posterior process. \figref{fig:decay-a} displays the relative error
$(L_n - l_n)/L_n$ as a function of data dimension $n$. We have also
included the graph of a reference function $h(n)=A^2/(8nb^4)$ for
empirical estimates of constants $A$ and $b$ as in
\propref{prop:lengths}. The difference between $(L_n - l_n)/L_n$ and
$h(n)$ is plotted in \figref{fig:decay-b}. This illustrates
\propref{prop:lengths} as the theory matches the empirical study.

\begin{figure}[ht]
\centering
\begin{subfigure}[]{0.45\columnwidth}
\centering \includegraphics[scale=0.4, clip, trim={0pt 0pt 0pt
    0pt}]{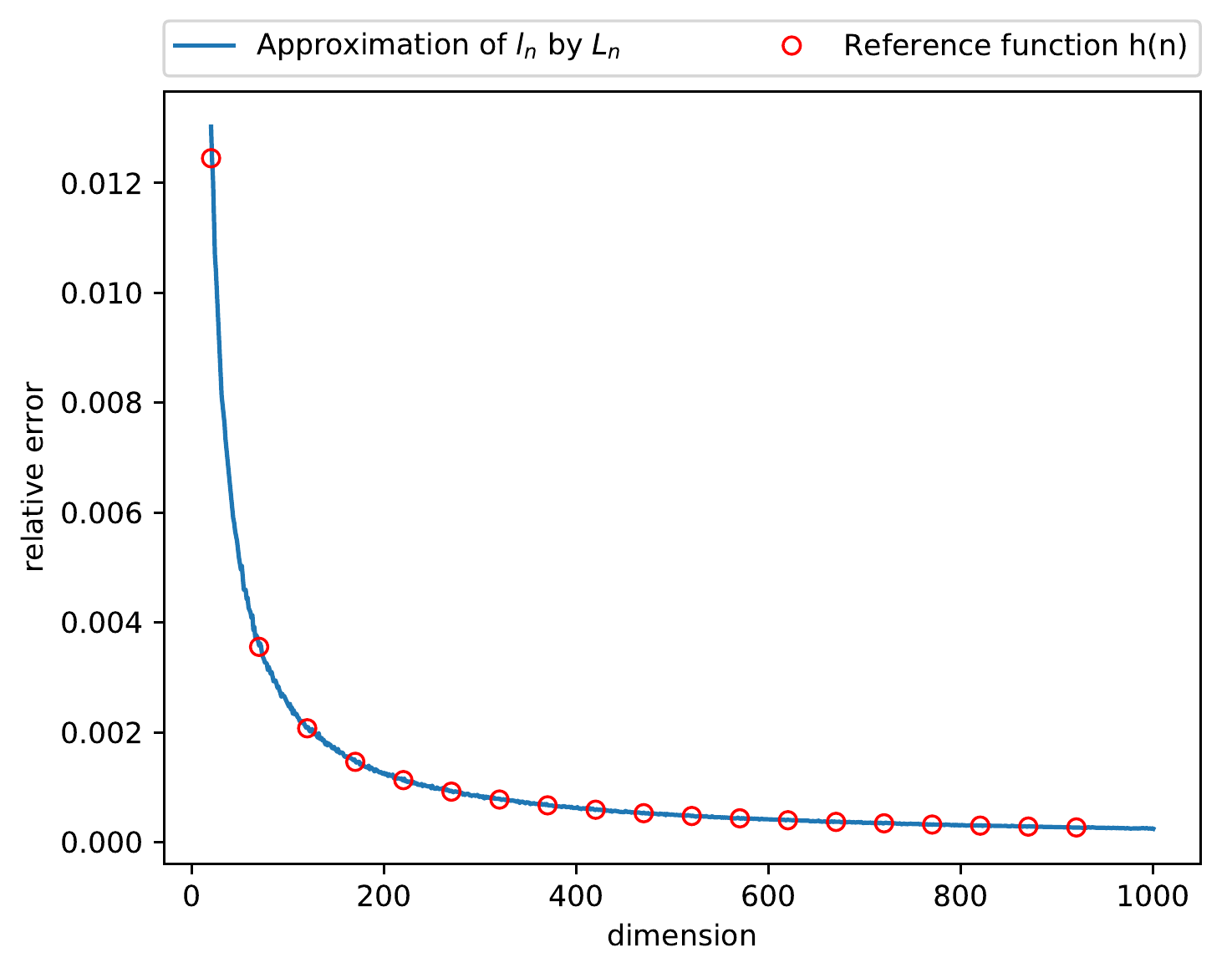}
\caption{Plots of $(L_n-l_n)/L_n$ and $h(n)$.}\label{fig:decay-a}
\end{subfigure}
\begin{subfigure}[]{0.4\columnwidth}
\centering \includegraphics[scale=0.4, clip, trim={0pt 0pt 0pt
    0pt}]{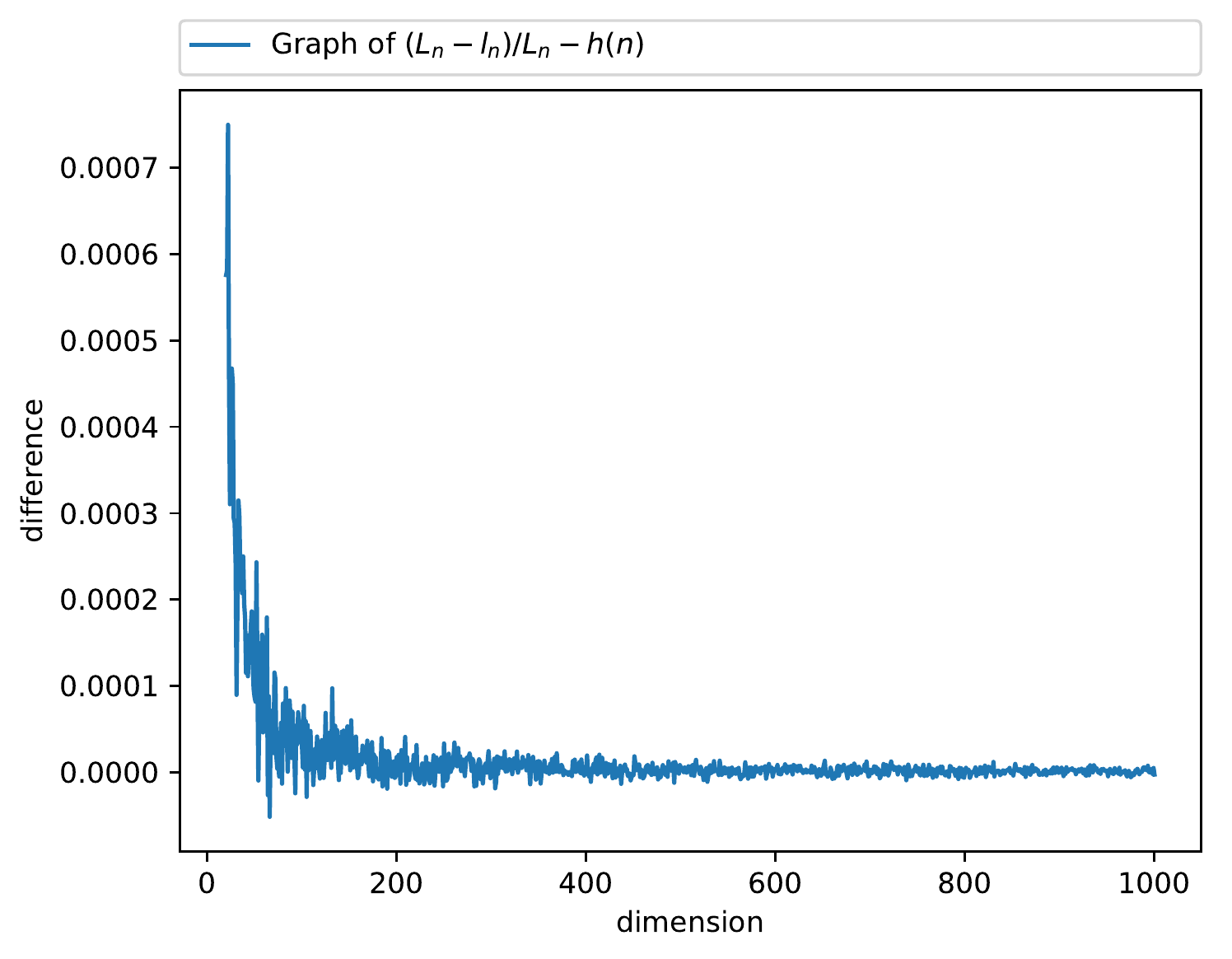}
\caption{Plot of $(L_n-l_n)/L_n - h(n)$.}\label{fig:decay-b}
\end{subfigure}
  \caption{}
\label{fig:decay}
\end{figure}

\end{example}

\section{Concluding remarks}
Starting from the goal of learning an \emph{operational
  representation}, we have studied a general class of generative
latent variable models $\x = f(\z)$, where $f$ is a smooth stochastic
process. The latent space can here be endowed with a random Riemannian
metric, such that elementary operations (interpolations, distances,
integration, etc.) can be defined in a way that is invariant to
reparametrizations of the model.

Mathematically, this is a natural approach, but it does not lend
itself easily to computations as computational tools do not exist for
working with random Riemannian manifolds. In this paper we have
provided a deterministic approximation to this large class of random
Riemannian metrics, and provided tight approximation bounds.  In
particular, it is worth noting that the bound is very tight when data
is high dimensional, which is the common case for machine
learning. Within this deterministic approximation, we can apply
standard tools for computations over Riemannian manifolds
\cite{hauberg:nips:2012, hennig:aistats:2014, Freifeld:CVPR:2014},
and thereby realize the idea of operational representation learning.

\section*{Acknowledgments}
This work was supported by a research grant (15334) from VILLUM FONDEN. This project has
received funding from the European Research Council (ERC) under the European
Union's Horizon 2020 research and innovation programme (grant agreement n\textsuperscript{o} 757360). 

\bibliography{collection}

\begin{thebibliography}{10}

\bibitem{AT07}
R.~J. Adler and J.~E. Taylor.
\newblock {\em Random Fields and Geometry}.
\newblock Springer, 2007 edition, 2007.

\bibitem{AHH18}
G.~Arvanitidis, L.~K. Hansen, and S.~Hauberg.
\newblock Latent space oddity: on the curvature of deep generative models.
\newblock {\em International Conference on Learning Representations (ICLR)},
  2018.

\bibitem{BW09}
R.~G. Baraniuk and M.~B. Wakin.
\newblock Random projections of smooth manifolds.
\newblock {\em Foundations of Computational Mathematics}, 9:51--77, 2009.

\bibitem{Belkin:2003:LED}
Mikhail Belkin and Partha Niyogi.
\newblock {Laplacian Eigenmaps for Dimensionality Reduction and Data
  Representation}.
\newblock {\em Neural Computation}, 15(6):1373--1396, June 2003.

\bibitem{BM13}
G.~Biau and D.~M. Mason.
\newblock High-dimensional p-norms.
\newblock {\em arXiv:1311.0587}, 2013.

\bibitem{blei2017variational}
David~M Blei, Alp Kucukelbir, and Jon~D McAuliffe.
\newblock Variational inference: A review for statisticians.
\newblock {\em Journal of the American Statistical Association},
  112(518):859--877, 2017.

\bibitem{B06}
Avrim Blum.
\newblock Random projection, margins, kernels, and feature-selection.
\newblock In Craig Saunders, Marko Grobelnik, Steve Gunn, and John
  Shawe-Taylor, editors, {\em Subspace, Latent Structure and Feature
  Selection}, pages 52--68. Springer Berlin Heidelberg, 2006.

\bibitem{C08}
Kenneth~L. Clarkson.
\newblock Tighter bounds for random projections of manifolds.
\newblock In {\em Proceedings of the Twenty-fourth Annual Symposium on
  Computational Geometry}, SCG '08, pages 39--48, New York, NY, USA, 2008. ACM.

\bibitem{C62}
H.~Cram{\'e}r.
\newblock {\em Mathematical methods of statistics}.
\newblock Asia Publishing House, 1962.

\bibitem{DG03}
Sanjoy Dasgupta and Anupam Gupta.
\newblock An elementary proof of a theorem of {Johnson} and {Lindenstrauss}.
\newblock {\em Random Structures \& Algorithms}, 22(1):60--65, 2003.

\bibitem{Donoho5591}
David~L. Donoho and Carrie Grimes.
\newblock Hessian eigenmaps: Locally linear embedding techniques for
  high-dimensional data.
\newblock {\em Proceedings of the National Academy of Sciences},
  100(10):5591--5596, 2003.

\bibitem{E69}
Carl-Gustav Esseen.
\newblock On the remainder term in the central limit theorem.
\newblock {\em Ark. Mat.}, 8(1):7--15, 11 1969.

\bibitem{Freifeld:CVPR:2014}
Oren Freifeld, S{\o}ren Hauberg, and Michael~J. Black.
\newblock Model transport: Towards scalable transfer learning on manifolds.
\newblock In {\em Proceedings IEEE Conf. on Computer Vision and Pattern
  Recognition (CVPR)}, Columbus, Ohio, USA, June 2014.

\bibitem{DFMV08}
Yoav Freund, Sanjoy Dasgupta, Kabra Mayank, and Nakul Verma.
\newblock Learning the structure of manifolds using random projections.
\newblock In J.~C. Platt, D.~Koller, Y.~Singer, and S.~T. Roweis, editors, {\em
  Advances in Neural Information Processing Systems 20}, pages 473--480. Curran
  Associates, Inc., 2008.

\bibitem{gallot1990riemannian}
Sylvestre Gallot, Dominique Hulin, and Jacques Lafontaine.
\newblock {\em Riemannian geometry}, volume~3.
\newblock Springer, 1990.

\bibitem{GPy}
{GPy}.
\newblock {GPy}: A {Gaussian} process framework in python.
\newblock \url{http://github.com/SheffieldML/GPy}, 2012.

\bibitem{H18}
S.~Hauberg.
\newblock Only bayes should learn a manifold (on the estimation of differential
  geometric structure from data).
\newblock {\em arXiv:1806.04994}, 2018.

\bibitem{hauberg:nips:2012}
S{\o}ren Hauberg, Oren Freifeld, and Michael~J. Black.
\newblock A geometric take on metric learning.
\newblock In P.~Bartlett, F.C.N. Pereira, C.J.C. Burges, L.~Bottou, and K.Q.
  Weinberger, editors, {\em Advances in Neural Information Processing Systems
  (NIPS) 25}, pages 2033--2041. MIT Press, 2012.

\bibitem{BHW08}
Chinmay Hegde, Michael Wakin, and Richard Baraniuk.
\newblock Random projections for manifold learning.
\newblock In J.~C. Platt, D.~Koller, Y.~Singer, and S.~T. Roweis, editors, {\em
  Advances in Neural Information Processing Systems 20}, pages 641--648. Curran
  Associates, Inc., 2008.

\bibitem{hennig:aistats:2014}
Philipp Hennig and S{\o}ren Hauberg.
\newblock Probabilistic solutions to differential equations and their
  application to {R}iemannian statistics.
\newblock In {\em Proceedings of the 17th international Conference on
  Artificial Intelligence and Statistics (AISTATS)}, volume~33, 2014.

\bibitem{K04}
R.~A. Khan.
\newblock Approximation for the expectation of a function of the sample mean.
\newblock {\em Statistics}, 38:117--122, 2004.

\bibitem{kingma:iclr:2014}
Diederik~P Kingma and Max Welling.
\newblock Auto-{E}ncoding {V}ariational {B}ayes.
\newblock In {\em Proceedings of the 2nd International Conference on Learning
  Representations (ICLR)}, 2014.

\bibitem{GGL16}
Subhaneil Lahiri, Peiran Gao, and Surya Ganguli.
\newblock Random projections of random manifolds.
\newblock {\em arXiv:1607.04331}, 07 2016.

\bibitem{L05}
Neil Lawrence.
\newblock Probabilistic non-linear principal component analysis with {Gaussian}
  process latent variable models.
\newblock {\em J. Mach. Learn. Res.}, 6:1783--1816, December 2005.

\bibitem{M05}
R.~J. Muirhead.
\newblock {\em Aspects of Multivariate Statistical Theory}.
\newblock Wiley, 2005.

\bibitem{Pennec:JMIV:06}
Xavier Pennec.
\newblock {Intrinsic Statistics on {R}iemannian Manifolds: Basic Tools for
  Geometric Measurements}.
\newblock {\em Journal of Mathematical Imaging and Vision}, 25(1):127--154,
  July 2006.

\bibitem{RW06}
C.~E. Rasmussen and C.~K.~I. Williams.
\newblock {\em Gaussian Processes for Machine Learning}.
\newblock The MIT Press, 2006.

\bibitem{rezende2014stochastic}
Danilo~Jimenez Rezende, Shakir Mohamed, and Daan Wierstra.
\newblock Stochastic backpropagation and approximate inference in deep
  generative models.
\newblock In Eric~P. Xing and Tony Jebara, editors, {\em Proceedings of the
  31st International Conference on Machine Learning}, volume~32 of {\em
  Proceedings of Machine Learning Research}, pages 1278--1286. PMLR, 2014.

\bibitem{lle}
Sam~T Roweis and Lawrence~K Saul.
\newblock Nonlinear dimensionality reduction by locally linear embedding.
\newblock {\em Science}, 290(5500):2323--2326, 2000.

\bibitem{Scholkopf99:kernelprincipal}
Bernhard Sch{\"o}lkopf, Alexander Smola, and Klaus-Robert Müller.
\newblock Kernel principal component analysis.
\newblock In {\em Advances in Kernel Methods - Support Vector Learning}, pages
  327--352, 1999.

\bibitem{isomap}
Joshua~B Tenenbaum, Vin De~Silva, and John~C Langford.
\newblock A global geometric framework for nonlinear dimensionality reduction.
\newblock {\em Science}, 290(5500):2319--2323, 2000.

\bibitem{LHTV14}
Alessandra Tosi, S{\o}ren Hauberg, Alfredo Vellido, and Neil~D. Lawrence.
\newblock Metrics for probabilistic geometries.
\newblock In {\em The Conference on Uncertainty in Artificial Intelligence
  (UAI)}, Quebec, Canada, July 2014.

\bibitem{V11}
N.~Verma.
\newblock A note on random projections for preserving paths on a manifold.
\newblock {\em UC San Diego, Tech. Report CS2011-0971}, 2011.

\end{thebibliography}

\end{document}